\documentclass[letterpaper, 10 pt, conference]{ieeeconf}  

\IEEEoverridecommandlockouts                              

\usepackage[hidelinks,colorlinks]{hyperref}
\hypersetup{
  colorlinks = true, 
  allcolors = blue,
  urlcolor = blue, 
  linkcolor = blue, 
  citecolor = blue, 
}
\usepackage{cite}
\usepackage[amssymb]{SIunits}
\usepackage{graphicx}
\usepackage{subcaption}
\usepackage{multirow}
\usepackage{amsmath,amssymb}
\newtheorem{lemma}{Lemma}
\usepackage{algorithm}
\usepackage{algpseudocode}
\let\labelindent\relax
\usepackage{enumitem}
\usepackage{balance}
\usepackage[misc]{ifsym}
\usepackage{booktabs}
\usepackage{lscape}

\usepackage{enumitem}
\setlist[itemize]{noitemsep,nolistsep,leftmargin=17pt}
\setlist[enumerate]{noitemsep,nolistsep,leftmargin=17pt}
\usepackage[font={small}]{caption}
\usepackage[colorinlistoftodos,prependcaption,textsize=tiny]{todonotes}
\newcommand{\fref}[1]{Fig.~\ref{#1}}
\newcommand{\sref}[1]{Section~\ref{#1}}
\newcommand{\tref}[1]{Table~\ref{#1}}

\begin{document}
\title{\LARGE \bf
iMTSP: Solving Min-Max Multiple Traveling\\Salesman Problem with Imperative Learning
}
\author{Yifan Guo, Zhongqiang Ren, and Chen Wang*
\thanks{Corresponding author. E-mail: \texttt{chenw@sairlab.org}. The source code, pre-trained model, and relevant data of this paper will be released at \url{https://github.com/sair-lab/iMTSP}}
\thanks{Yifan Guo is with the Flight Dynamics \& Control/Hybrid Systems Lab, Purdue University, West Lafayette, IN 47907. Email: \texttt{guo781@purdue.edu}}
\thanks{Zhongqiang Ren is with Robotics Institute, Carnegie Mellon University, Pittsburgh, PA 15213, USA. Email: \texttt{zhongqir@andrew.cmu.edu}}
\thanks{Chen Wang is with Spatial AI \& Robotics (SAIR) Lab, Department of Computer Science and Engineering, University at Buffalo, NY 14260, USA.}
}%

\maketitle
\thispagestyle{empty}
\pagestyle{empty}

\begin{abstract}
This paper considers a Min-Max Multiple Traveling Salesman Problem (MTSP), where the goal is to find a set of tours, one for each agent, to collectively visit all the cities while minimizing the length of the longest tour.
Though MTSP has been widely studied, obtaining near-optimal solutions for large-scale problems is still challenging due to its NP-hardness.
Recent efforts in data-driven methods face challenges of the need for hard-to-obtain supervision and issues with high variance in gradient estimations, leading to slow convergence and highly sub-optimal solutions.
We address these issues by reformulating MTSP as a bilevel optimization problem, using the concept of imperative learning (IL). This involves introducing an allocation network that decomposes the MTSP into multiple single-agent traveling salesman problems (TSPs). The longest tour from these TSP solutions is then used to self-supervise the allocation network, resulting in a new self-supervised, bilevel, end-to-end learning framework, which we refer to as imperative MTSP (iMTSP). Additionally, to tackle the high-variance gradient issues during the optimization, we introduce a control variate-based gradient estimation algorithm.
Our experiments showed that these innovative designs enable our gradient estimator to converge $20\times$ faster than the advanced reinforcement learning baseline, and find up to $80\%$ shorter tour length compared with Google OR-Tools MTSP solver, especially in large-scale problems (e.g. $1000$ cities and $15$ agents).
\end{abstract}

\section{Introduction}

The multiple traveling salesman problem (MTSP) seeks tours for multiple agents such that all cities are visited exactly once while minimizing an objective function defined over the tours.
MTSP arises in numerous robotics topics that require a team of robots to collectively visit many target locations such as unmanned aerial vehicles path planning~\cite{sundar2013algorithms, ma2019coordinated}, automated agriculture~\cite{carpio2021mp}, warehouse logistics~\cite{10109784}.
As their names suggest, Min-Sum MTSP minimizes the sum of tour lengths, while Min-Max MTSP minimizes the longest tour length, both of which are NP-hard~\cite{bektas2006multiple} to solve to optimality.\footnote{MTSP is challenging and renowned for its NP-hardness~\cite{bektas2006multiple}.
Intuitively, MTSP involves many decision variables: both assigning the cities to the agents and determining the visiting order of the assigned cities. As an example, an MTSP with \(100\) cities and \(10\) agents involves \(\frac{100! \times 99!}{10! \times 89!} \approx 10^{20000}\) possible solutions, while the number of atoms in the observable universe is estimated to be ``merely'' in the range from \(10^{78}\) to \(10^{82}\)~\cite{gaztanaga2023mass}.}
In this paper, we focus on Min-Max MTSP while our method will also be applicable to Min-Sum MTSP. Min-Max MTSP is more often used when the application seeks to minimize the overall completion time~\cite{park2023learn}.
Due to the NP-hardness of MTSP, one must consider the trade-off between the solution optimality and the computational efficiency, since finding an optimal solution is often intractable for large problems.

A variety of classic (non-learning) approaches, including exact~\cite{ham2018integrated,ren2021ms}, approximation~\cite{berczi2023approximations}, and heuristic~\cite{al2019comparative,necula2015tackling,ma2019coordinated} algorithms, have been developed to handle MTSP, but most of them consider minimizing the sum of tour lengths (Min-Sum), rather than the maximum tour length (Min-Max).
In recent years, there has been a notable shift towards employing deep learning techniques to tackle TSP and MTSP~\cite{liang2023splitnet,hu2020reinforcement, xin2021neurolkh,miki2018applying, vinyals2015pointer,kool2018attention,nazari2018reinforcement,park2021schedulenet, khalil2017learning, wu2021learning}.
However, these methods still have fundamental limitations, particularly in their ability to generalize to unseen problem sizes, and in consistently finding high-quality solutions for large-scale problems.
For deep learning-based methods that supervise the model with heuristic or exact solutions~\cite{xin2021neurolkh,miki2018applying, vinyals2015pointer}, they struggle with limited supervision on small-scale problems and lack feasible supervision for large-scale instances, leading to poor generalization. Reinforcement learning (RL)-based approaches~\cite{hu2020reinforcement,kool2018attention,nazari2018reinforcement,park2021schedulenet} usually exploit implementations of the policy gradient algorithm, such as the REINFORCE~\cite{williams1992simple} algorithm and its variants. These RL approaches face the issue of high-variance gradient estimations, which can result in slow convergence and highly sub-optimal solutions. Researchers have also developed strategies like training greedy policy networks~\cite{nazari2018reinforcement, kool2018attention, khalil2017learning} or iteratively improving solutions~\cite{ wu2021learning}, which often get stuck at local optima.

To enhance generalization and accelerate the training process, we reformulate the Min-Max MTSP as a bilevel optimization problem. This comprises an upper-level optimization that focuses on training a city allocation network, and a lower-level optimization that solves multiple signal-agent TSPs, inspired by imperative learning (IL)~\cite{fu2023islam, yang2023iplanner,zhan2023imatching}.
Specifically, the allocation network is to assign each city to an agent, which decomposes the MTSP into multiple single-agent TSPs.
Then, a classic TSP solver, which can quickly produce a near-optimal tour for a single-agent TSP with up to a few hundred cities, is employed to find a tour for each agent based on the allocation. 
A key aspect is that the upper-level optimization, namely the training of the allocation network, is supervised by the lower-level TSP solver. This metric-based supervision from the lower-level optimization results in an end-to-end, self-supervised framework for solving MTSP, which we refer to as imperative MTSP (iMTSP).

One technical obstacle when developing iMTSP is that the space of possible city allocations is discrete.
As a result, the objective function, i.e., the maximum tour length given by the lower-level TSP solver, is non-differentiable to the city allocation, which prevents the back-propagation to the allocation network's parameters.
One of the solutions is to explore the policy gradient algorithm or the reparameterization trick~\cite{figurnov2018implicit}, which often use Monte-Carlo methods to estimate the gradient. However, these methods can lead to gradient estimations with high variance since the space of city allocation is extremely large.
To handle this difficulty, we introduce a control variate~\cite{nelson1990control} to reduce the variance of the gradient estimations.
Parameterized by a trainable surrogate network running in parallel with the TSP solver, the control variate can be efficiently computed, and the overall algorithm can provide low variance gradient information through the non-differentiable solver and the discrete decision space.
The main contributions of this paper are summarized as follows:
\begin{enumerate}[leftmargin=0.3cm]
    \item \textbf{Framework}: We formulate the MTSP as a bilevel optimization problem, introduce a new end-to-end, self-supervised framework for Min-Max MTSP. This decomposes a large-scale Min-Max MTSP into several smaller TSPs, each of which can be solved efficiently.
    \item \textbf{Methodology}: To tackle the problem of bilevel optimization in discrete space, we introduce a control variate-based technique for back-propagation.
    This produces low-variance gradient estimations despite the non-differentiability of the lower-level optimization problem.
    \item \textbf{Experiments}: We corroborate iMTSP in terms of efficiency and solution quality against both an RL-based and a classic (non-learning) methods.
    On large-scale problems, iMTSP achieves at most $80\%$ shorter tours and with only $1.6\%$ inference time compared with Google OR-Tools routing library~\cite{ORTools_options}.
    Also, iMTSP's gradient estimator converges about $20\times$ faster and produces $3.2 \pm 0.01\%$ shorter tours, compared with the RL-based approach~\cite{hu2020reinforcement}.
\end{enumerate}
\section{Related works}\label{sec:related}
MTSP has been extensively studied. Analytical methods~\cite{ham2018integrated} use graph theory, game theory, and combinatorial optimization, etc.~\cite{cheikhrouhou2021comprehensive}, and can provide guarantees on solution optimality. However, these methods usually struggle at solving large-scale MTSP due to its NP-hardness.

Heuristic approaches for MTSP handle the challenge by decomposing the original problem into several phases which in turn reduces the computational burden of each phase.
Examples are particle swarm optimization~\cite{wei2020particle}, ant colony optimization~\cite{necula2015tackling}, etc~\cite{kitjacharoenchai2019multiple,  murray2020multiple}. 
Although heuristic algorithms usually run faster than analytical approaches, they lack theoretic guarantees on solution quality and could produce highly sub-optimal solutions especially for large-scale problems with hundreds and thousands of cities. 

Recently, researchers tried to solve MTSP with machine learning (ML) methods using recurrent neural networks (RNN)~\cite{rumelhart1985learning}, transformer~\cite{vaswani2017attention}, etc. These learning-based methods usually scale well as the size of the problem increases. However, they may not generalize well to problem sizes outside the training set. 
As an early attempt~\cite{vinyals2015pointer}, an RNN based method is developed and can handle TSP with up to $50$ cities.
Hu et al.~\cite{hu2020reinforcement} propose a hybrid approach within the RL framework, which shows impressive generalization ability.
However, their gradient estimator has high variance, which leads to slow convergence and sub-optimal solutions especially when dealing with large-scale problems. 
Other ML-based methods include~\cite{nazari2018reinforcement, kool2018attention, khalil2017learning, wu2021learning}. Note that most of them consider TSP instead of MTSP.

\noindent\textbf{Imperative Learning} As an emerging learning framework, imperative learning (IL) has inspired some pioneering works in simultaneous localization and mapping (SLAM)~\cite{fu2023islam}, path planning~\cite{yang2023iplanner}, and visual feature matching~\cite{zhan2023imatching}. The most significant difference between our work and theirs is that the lower-level optimization problems in their scenarios are all in continuous space, while iMTSP needs to deal with an extremely large discrete decision space.

\noindent\textbf{Control Variate} Control variate has been investigated by the ML community.
In general, any policy gradient algorithm with a baseline or an actor-critic can be interpreted as an additive control variate method~\cite{greensmith2004variance}. Additionally, Grathwohl et.al.~\cite{grathwohl2017backpropagation} propose gradient estimators for a non-differentiable function with continuous and discrete inputs, respectively. Their insight is evaluating the control variate directly with a single sample estimation of the gradient variance. Our work leverages this insight. Gu et.al.~\cite{gu2017interpolated} produce a family of policy gradient methods using control variate to merge gradient information from both on-policy and off-policy updates. Control variate technique has also been used for meta-RL~\cite{liu2019taming}, federated learning~\cite{karimireddy2020scaffold}, etc~\cite{baker2019control,geffner2018using, guo2021multi}.
\section{Problem Formulation}\label{sec:formulation}

Given a set of $N$ cities collectively denoted as $X = \{\mathbf{x_1}, ..., \mathbf{x_N}\}$, along with $M$ agents starting from a common city, i.e., the depot $\mathbf{x_d} \in X$, the MTSP is characterized by the tuple $(X, \mathbf{x_d}, M)$. The objective is to find, for each agent $j$, a tour $T_j = \{\mathbf{x}_{j,1},..., \mathbf{x}_{j, k},...\}$ such that all cities are visited exactly once, each tour returns to the depot, and the length of the longest tour reaches the minimum. The optimization objective for Min-Max MTSP can be defined as:
\begin{equation}
   \min_{T_{1,\dots,M}} \max_jD(T_j),
\end{equation}
where $D(T_j)$ is the cost of route $T_j$. For easier understanding, we use Euclidean distance as the cost in this paper, i.e.,:
\begin{equation}\label{cost}
       D(T_j)=\sum_{k=1}^{\lVert T_j \rVert} \|\mathbf{x}_{j,k+1} - \mathbf{x}_{j,k}\|_2,\quad j=1,2,\cdots,M,
\end{equation}
where $\mathbf{x}_{j,k}$ denotes the city coordinate that is the $k$-th destination of the $j$-th agent.

\section{Methods}\label{sec:methods}

This section first elaborates on the reformulation of MTSP as a bilevel optimization problem, and then introduces the forward computation of the our framework and the control variate-based optimization algorithm.

\subsection{Reformulation of MTSP}
The MTSP formulation in \sref{sec:formulation} is reformulated as the following bilevel optimization problem: 
For each instance $(X,\mathbf{x_d},M)$, the upper-level problem is to assign each city to an agent, and the lower-level problem is to compute the visiting order of assigned cities for each agent. We define $f$ as the allocation network with parameters $\mathbf{\theta}$, and $g$ as the single-agent TSP solver with parameters $\mathbf{\mu}$. The imperative MTSP (iMTSP) is defined as:
\begin{subequations}\label{eq:imtsp}
\begin{align}
    & \min_{\theta} U_\theta \left(L(\mu^*), f(\theta); \gamma\right);\quad \min_\gamma U_\gamma\left(L(\mu^*), s(\gamma); \theta\right)\\
    & \textrm{s.t.}\quad \mathbf{\mu}^* = \arg\min_\mu L(\mu),\\
    & \quad \quad~ L \doteq \max_j D(g(a_j;\mu)), \quad j=1,2,\cdots,M, \\
    & \quad \quad~ a_j \sim f(\theta),
\end{align}
\end{subequations}

where $a_j$ is the allocation for agent $j$ sampling from the probability matrix by row, and $L$ returns the maximum route length among all agents. $s$ is a surrogate network with parameter $\gamma$ , which will be explained in detail in \sref{sec:optimization} along with the specific definitions of upper-level objectives.
Specifically, the upper-level optimization consists of two separate steps: the upper-level cost $U_\theta$ is to optimize the allocation network $f(\theta)$, while $U_\gamma$ is to optimize the surrogate network $s(\gamma)$. The framework of iMTSP is shown in \fref{fig:multi architecture}.

\begin{figure}[t]
	\centering
    \includegraphics[width=\linewidth]{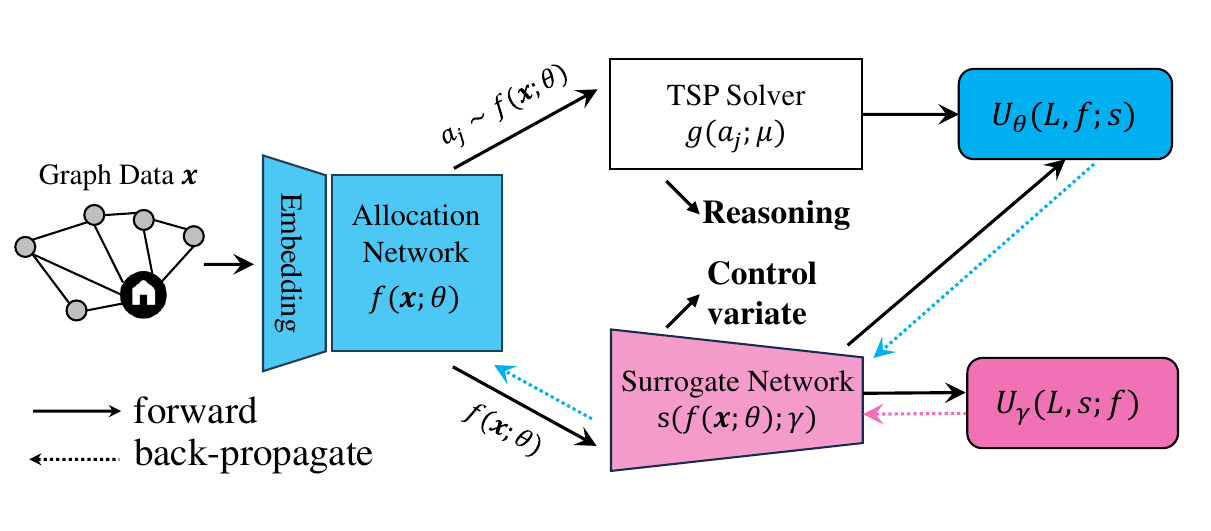}
    \caption{The framework of our self-supervised MTSP network. The allocation network uses supervision from the TSP solver, and the surrogate network is supervised by the single sample variance estimator, reducing the gradient variance.}
	\label{fig:multi architecture}
\end{figure}

\subsection{Forward Propagation Details}

We next present the details of solving iMTSP \eqref{eq:imtsp}. For each instance with coordinates $X$, we first embed agents and cities into feature vectors that contain the agent-city and city-city topological information. The embedded instance $X^\prime$ is the input to the allocation network $P = f(X^\prime,\mathbf{x_d},M; \mathbf{\theta})$. Each element $p_{i,j}$ in the output $P \in \mathbb{R}^{N \times M }$ is the probability that the $i$-th city is allocated to the $j$-th agent. This probability matrix is then sampled in each row so that each city is assigned to an agent. The sampling result is then reorganized into allocations $A=\{a_1, a_2,...,a_M\}$, where $a_j$ represents the set of cities assigned to the $j$-th agent. Subsequently, a TSP solver $g(a_j; \mathbf{\mu})$ for each agent $j$ is invoked to order the cities in the allocation $a_j$ and return a tour. Finally, the cost of each tour is computed with \eqref{cost} and the length of the longest tour is the cost of that instance, i.e., $L=\max_jD(g(a_j;\mu))$. The internal process of embedding and the allocation network is below.

\subsubsection{Embedding}
To extract the topological information from the coordinate data, we use a compositional message passing neural network (CMPNN)~\cite{gilmer2017neural} to embed each city and agent as a feature vector. First, the feature of the city $i$ is iteratively updated using its feature $\mathbf{f}_i^t$, features of its neighbors $\mathbf{f}_k^t$, and the weights $\mathbf{e}_{i,k}$ on the connecting edges. Here, $i$ and $k$ indicate different cities, and $t$ is the number of iterations of graph aggregation (instead of the training loop). The embedding of the entire graph $\mathbf{f}_g$ comes from all cities' features. Finally, a context feature vector $\mathbf{f}_c$ is formed as a concatenation of the graph feature $\mathbf{f}_g$ and the feature of the depot $\mathbf{f}_d$. Formally speaking, the embedding process is:
\begin{subequations}
\begin{align}
    \mathbf{m}_{i,k}^t &= \Omega(\mathbf{f}_i^t,\mathbf{f}_k^t,\mathbf{e}_{i,k}),\\
    \mathbf{l}_i^t &= \Phi _{k \in \mathcal{N}_i(\mathbf{m}_{i,k}^t)},\\
    \mathbf{f}_{i}^{t+1}&=\Psi (\mathbf{f}_i^t, \mathbf{l}_k),\\
    \mathbf{f}^{t}_{g}&=\Phi^\prime (\mathbf{f}_i^t),\quad i={2,3,...,n},\\
    \mathbf{f}_c &= [\mathbf{f}_g;\mathbf{f}_d],
\end{align}
\end{subequations}
where functions $\Omega$ and $\Psi$ are represented by neural networks; $\Phi$ and $\Phi^\prime$ are element-wise nonlinear functions; $\mathcal{N}_i$ contains all the neighbours of city $i$; $\mathbf{m}_{i,k}^t$ is the processed information from city $k$ to city $i$; and $\mathbf{l}_i^t$ represents aggregated information from all the neighbours of city $i$.
After obtaining the context feature $\mathbf{f}_c$, the embedding of agent $j$ is calculated using the attention mechanism as follows:
\begin{subequations}
\begin{align}
    \mathbf{k}_{j,i} &= \mathbf{\theta_k}\mathbf{f}_i,\quad \mathbf{q}_{j} = \mathbf{\theta_q} \mathbf{f}_c, \quad i=2,3,...,n\\
    \mathbf{v}_{j,i} &= \mathbf{\theta_{v}} \mathbf{f}_i,\quad i=2,3,...,n\\
     u_{j,i} &= \frac{\mathbf{q}^{\intercal}_j \mathbf{k}_{j,i}}{\sqrt{d_{\mathbf{k}}}},\quad i=2,3,...,n\\
     w_{j,i}&=\frac{e^{u_{j,i}}}{\sum_k e^{u_{j,k}}},\quad i={2,3,...,n};~k={2,3,...,n}\\
     \mathbf{h}_j &= \sum_i w_{j,i}\mathbf{v}_{j,i},
\end{align}
\end{subequations}
where $\mathbf{k}$, $\mathbf{q}$ and $\mathbf{v}$ are respectively the keys, queries, and values of the attention layer. $d_{\mathbf{k}}$ is the dimension of keys, and $i$ and $k$ are used to index cities. All the $\theta$ are trainable parameters.
The keys and values are first computed from the city embedding, and the queries are from the context feature $\mathbf{f}_c$. Then, we calculate the attention weight $w_{j,i}$ of each agent-city pair. Finally, the agent embedding $\mathbf{h}_j$ is the sum of all the city values, weighted by the attention weights.

\subsubsection{Allocation Network}

As we have extracted the topological information into vector representations, we can now calculate the allocation matrix using the attention mechanism again. The keys $\mathbf{k}^\prime$ and queries $\mathbf{q}^\prime$ are computed from the city and agent embeddings, respectively, and the relative importance of each city to each agent is computed as:
\begin{subequations}
\begin{align}
    \mathbf{k}^\prime_{j,i} &= \mathbf{\theta_{k^\prime}} \mathbf{f}_i, \quad  \mathbf{q}^\prime_{j} = \mathbf{\theta_{k^\prime}} \mathbf{h}_j, \quad i=2,3,...,n\\
    u^\prime_{j,i} &= \frac{\mathbf{q}^{\prime\intercal}_j \mathbf{k}^\prime_{j,i}}{\sqrt{d_{\mathbf{k^\prime}}}},\quad i=2,3,...,n\\
    \beta_{j,i}&=\alpha \tanh(u^\prime_{j,i}),\label{imp}
\end{align}
\end{subequations}
where $d_{\mathbf{k^\prime}}$ is the dimension of new keys. Note that queries only depend on the agent embedding but not the city's. The notation $u^\prime_{j,i}$ represents the non-clipped importance scores, $\alpha$ is a clip constant, and $\beta_{j,i}$ is the clipped importance score of city $i$ to agent $j$, which will be used to decide which agent is assigned to visit the city. The probability of the city $i$ is visited by agent $j$ is calculated by applying the softmax function on the importance over all agents:
\begin{equation}
    p_{j,i}=\frac{e^{\beta_{j,i}}}{\sum_j e^{\beta_{j,i}}}.
\end{equation}
The allocation matrix has a shape of $N \times M$ with $N$ and $M$ are the number of cities and agents, respectively. The model produces the final allocation by sampling each row of the allocation matrix.
With the allocation, the MTSP is decomposed into $M$ TSPs. A classic TSP solver (such as Google OR-Tools) is invoked on each of the TSP to find the tour length of each TSP. The final return $L$ of the MTSP is the maximum tour length among all the TSPs. Because the TSP solver is non-differentiable and the allocation variables are discrete, network optimization is very challenging.

\subsection{Optimization}\label{sec:optimization}
Now we demonstrate how the non-differentiabiliy prevents us from using the analytical gradient, and how our control-variate approach can pass low-variance gradient estimations through the classic solver and the discrete allocation space.
\subsubsection{Gradient}
Directly taking the derivative of the objective $L=\max_jD(g(a_j;\mu))$ w.r.t. the neural parameter $\theta$, we have the analytical form of the gradient:
\begin{equation}
    \nabla_\theta L = \frac{\partial L}{\partial f} \frac{\partial f}{\partial \mathbf{\theta}} +  \frac{\partial L}{\partial g} \frac{\partial g}{\partial \mathbf{\mu}^*}{\frac{\partial \mathbf{\mu}^*}{\partial \mathbf{\theta}}}.
\end{equation}
There are two major difficulties prevent us from directly deploying this equation to update $\mathbf{\theta}$. First, the term $\frac{\partial L}{\partial f} \frac{\partial f}{\partial \mathbf{\theta}}$ is always zero since the objective $L$ does not directly depend on $f$. Thus, this term cannot provide any information to update $\theta$. The second term cannot be directly computed because: (i) for $\frac{\partial L}{\partial g}$, the $\max$ function used to select the longest tour length is not differentiable; (ii) for $\frac{\partial g}{\partial \mu^*}$ and ${\frac{\partial \mu^*}{\partial \theta}}$, the optimization process in $g(\cdot)$ is not differentiable; and (iii) $\mu^*$ and $\theta$ related implicitly.
To deal with the non-differentiability, we can use the Log-derivative trick, which has been widely used in RL, and estimates the above gradient as
\begin{equation}
     \nabla_\theta L =L\frac{\partial }{\partial \theta}\log P(\theta).
\end{equation}
However, our experiments showed that this estimator provides a gradient with a high variance, which slows down the convergence rate and compromises the solution quality.

\subsubsection{Control Variate}
To address the above issue, we introduce a control variate~\cite{nelson1990control} to reduce the variance of a stochastic estimator. Suppose $d$ (as in derivative) is a random variable (RV), $h(d)$ is an estimator of $\mathbb{E}(d)$, the control variate $c(d)$ is a function of $d$ whose mean value is known. Define a new estimator with a control variate as:
\begin{equation}\label{eq:control-variate}
    h(d)_{\text{new}} = h(d) + \zeta(c(d) - \mathbb{E}[c(d)]).
\end{equation}
When constant $\zeta$ is properly chosen, $h(d)_{new}$ has the same expected value but lower variance than $h(d)$, as long as the control variate $c(d)$ is correlated to $h(d)$.
We briefly summarize the underlying principle of control variate technique~\cite{meyn2008control} as follows and omit the dependency on $d$ for simplicity.

\begin{algorithm}[t]
\caption{The iMTSP algorithm}\label{alg:iMTSP}
\begin{algorithmic}

\Require The coordinate data $X$ and embedded data $X^\prime$, allocation network $f(X^\prime;\theta)$, TSP solver $g(A;\mu)$, surrogate network $s(P(\theta);\gamma)$, learning rates $\alpha_1$, $\alpha_2$
\While {not converged}
    \State $P \gets f(X^\prime;\theta)$
    \State $A \gets$ Sampling $P$ by row
    \State $C \gets g(A;\mu)$, $C^\prime \gets s(P(\theta);\gamma)$
    \State $ \frac{\partial}{\partial \theta} U_\theta \gets [L-L^\prime]\frac{\partial }{\partial \theta}\log P(\theta)+\frac{\partial }{\partial \theta}s(P(\theta);\gamma)$
    \State $\frac{\partial}{\partial \phi} U_\gamma \gets \frac{\partial}{\partial \gamma} (\frac{\partial}{\partial \theta}U_\theta)^2$
    \State $\theta \gets \theta-\alpha_1\frac{\partial}{\partial \theta} U_\theta$
    \State $\gamma \gets \gamma-\alpha_2\frac{\partial}{\partial \gamma}  U_\gamma$
\EndWhile\\
\Return $\theta$
\end{algorithmic}
\end{algorithm}

\begin{lemma}
Let $h$ denotes an RV with an \textbf{unknown} expected value $\mathbb{E}(h)$, and $c$ denotes an RV with a known expected value $\mathbb{E}(c)$. The new RV $h_{new}=h+\zeta(c-\mathbb{E}(c))$ has same expected value but smaller variance, when the constant $\zeta$ is properly chosen and $c$ is correlated with $h$.
\end{lemma}
\begin{proof}
The variance of the new estimator is
    \begin{equation}
        \mathrm{Var}(h_{\text{new}})=\mathrm{Var}(h)+\zeta^2\mathrm{Var}(c)+2\zeta \mathrm{Cov}(h,c).
    \end{equation}
Differentiate the above equation w.r.t. $\gamma$ and set the derivative to zero, the optimal choice of $\zeta$ is
    \begin{equation}
        \zeta^*=-\frac{\mathrm{Cov}(h,c)}{\mathrm{Var}(c)}.
    \end{equation}
The minimum variance of $h_{\text{new}}$ with $\zeta^*$ is $(1-\rho^2_{h,c})\mathrm{Var}(h)$ where $\rho$ is the correlation coefficient of $h$ and $c$. Since $\rho\in[-1,1]$, the inequality $\mathrm{Var}(h_{\text{new}})\leq \mathrm{Var}(h)$ always holds with $\zeta=\zeta^*$. When $\lVert \rho \rVert\rightarrow 1$, $\mathrm{Var}(h_{\text{new}})\rightarrow 0$.
\end{proof}

\begin{table*}[t]
  \centering
    \caption{Performance comparison on test MTSPs. The boldface denotes the best results on each training size. iMTSP provides solutions with overall $3.2 \pm 0.01\%$ shorter maximum route length comparing with the RL baseline, and has evidently advantages over the Google OR-Tools. The last column indicates how much longer are the baseline's routes than iMTSP's routes, averaging on all test size.}
  \begin{tabular}{ccc|ccccccc|c}
    \toprule
    \multicolumn{3}{c|}{Training Setting} & \multicolumn{7}{c|}{\# Test Cities}& \multirow{2}{*}{Avg gap}\\
    \cmidrule{1-10}
    Model& \# Agents & \# Cities & 400 & 500 & 600 & 700 & 800 & 900 & 1000\\
    \midrule
         Google OR-Tools~\cite{ORTools_options}& 10 & N.A. & 10.482 & 10.199 & 12.032 & 13.198 & 14.119 & 18.710 & 18.658&277.1\%\\
         Reinforcement Learning~\cite{hu2020reinforcement} &10& 50& 3.058 & 3.301 & 3.533 & 3.723 & 3.935 & \textbf{4.120} & \textbf{4.276}&0.2\%\\
         iMTSP (Ours) &10& 50&\textbf{3.046}&\textbf{3.273}&\textbf{3.530}&\textbf{3.712}&\textbf{3.924}&4.122&4.283&0\\
         \midrule
         Google OR-Tools~\cite{ORTools_options}&10&N.A.& 10.482 & 10.199 & 12.032 & 13.198 & 14.119 & 18.710 & 18.658&290.1\%\\
         Reinforcement Learning~\cite{hu2020reinforcement}&10& 100& 3.035 & 3.271 & 3.542 & 3.737 & 3.954 & 4.181 & 4.330&4.3\%\\
         iMTSP (Ours) &10& 100 &\textbf{3.019}&\textbf{3.215}&\textbf{3.424}&\textbf{3.570}&\textbf{3.763}&\textbf{3.930}&\textbf{4.042}&0\\
         \midrule
         Google OR-Tools~\cite{ORTools_options}&15&N.A.& 10.293 & 10.042 & 11.640 & 13.145 & 14.098 & 18.471 & 18.626&349.9\%\\
         Reinforcement Learning~\cite{hu2020reinforcement}&15& 100& 2.718 &2.915&3.103&3.242&3.441&3.609&3.730&6.3\%\\
         iMTSP (Ours) &15& 100&\textbf{2.614}&\textbf{2.771}&\textbf{2.944}&\textbf{3.059}&\textbf{3.221}&\textbf{3.340}&\textbf{3.456}&0\\
         \bottomrule
  \end{tabular}
  \label{tab:table2}
\end{table*}

In practice, $\zeta^*$ is usually inaccessible, and thus we cannot arbitrarily choose $c$. Without losing generality, consider the case where $h$ and $c$ are positively correlated. By solving the inequality $\zeta^2\mathrm{Var}(c)+2\zeta\mathrm{Cov}(h,c)<0$, we get $\zeta \in (\frac{-2\mathrm{Cov}(h,c)}{\mathrm{Var}(c)},0)$, which indicates that as long as $k$ is set in this range, $\mathrm{Var}(h_{\text{new}})$ is less than $\mathrm{Var}(h)$. When $h$ and $c$ are highly correlated and are of similar ranges, $\mathrm{Cov}(h,c) \approx \mathrm{Var}(c)$, thus $\zeta^* \approx -1$, which is our empirical choice of $\zeta$.

Inspired by the control variant gradient estimator \eqref{eq:control-variate} we design a lower variance, unbiased gradient estimator using trainable surrogate network $s$. Choose the allocation log-derivative estimator as $h(d)$ in (\ref{eq:control-variate}), and replace $c(d)$ and $\mathbb{E}[c(d)]$ with the surrogate log-derivative estimator and the actual surrogate network gradient, respectively, the upper-level costs $U_\theta$ of our bilevel framework is:
\begin{equation}
\label{eq:upper-theta}
    U_\theta \doteq [L(\mu^*) -L^\prime]\log P(\theta)+s(P(\theta);\gamma),
\end{equation}
where the constant $\zeta$ in (\ref{eq:control-variate}) is substituted with $-1$. $L^\prime$ is output of the surrogate network (control variate) but is taken as a constant by cutting its gradients.
In turn, our gradient for the allocation network parameters $\theta$ is:
\begin{equation}
    \frac{\partial U_\theta}{\partial \theta}=[L(\mu^*)-L^\prime]\frac{\partial }{\partial \theta}\log P(\theta)+\frac{\partial }{\partial \theta}s(P(\theta);\gamma).
\end{equation}
\subsubsection{Surrogate Network}
We implement the idea of control variate with a surrogate network $s(\gamma)$ in parallel with the TSP solver. Note that the surrogate network takes the allocation matrix $P$ as its input and predicts the maximum tour length, while the inputs of the actual TSP solver are allocations $A$ and coordinate data $X$. Naively optimizing the surrogate network to imitate the input-output map of the TSP solver is ineffective because learning such a map does not guarantee learning correct gradient information. Since we require the surrogate network (control variate) to reduce the gradient variance of allocation parameters, we directly optimize the surrogate parameters to minimize a single sample estimation of the allocation gradient variance. Thus, the upper-level objective $U_\gamma$ for the surrogate network is
\begin{equation}
        U_\gamma \doteq \rm{Var}( \frac{\partial U_\theta}{\partial \theta}).
\end{equation}

Using the fact $\rm{Var}( \frac{\partial U_\theta}{\partial \theta})= \mathbb{E}[(\frac{\partial U_\theta}{\partial \theta})^2] - \mathbb{E}[(\frac{\partial U_\theta}{\partial \theta})]^2$, the gradient of the surrogate network is
\begin{equation}\label{eq:grad-gamma}
\frac{\partial U_\gamma}{\partial \gamma} =
    \frac{\partial}{\partial \gamma}\mathbb{E}[ (\frac{\partial U_\theta}{\partial \theta})^2]-\frac{\partial}{\partial \gamma}\mathbb{E}[ \frac{\partial U_\theta}{\partial \theta}]^2 = \frac{\partial}{\partial \gamma}\mathbb{E}[ (\frac{\partial U_\theta}{\partial \theta})^2],
\end{equation}
The term $\frac{\partial}{\partial \gamma}\mathbb{E}[ (\frac{\partial U_\theta}{\partial \theta})^2]$ is always zero because the surrogate network has no explicit influence on the return $L$.
For ease of understanding, we summarize our control variate-based bilevel optimization method in Algorithm \ref{alg:iMTSP}.

\section{Experiments}\label{sec:experiments}

We next conduct comprehensive experiments to demonstrate the generalization ability, computational efficiency, and convergence speed of our iMTSP framework. We focus our experiments on large-scale settings ($\geq 400$ cities) because current classic solvers can already efficiently find near-optimal solutions for small-scale problems.

\subsection{Baselines}

To demonstrate the advantages of our method, we compare it with both classic methods and learning-based methods.
Specifically, we select a state-of-the-art reinforcement learning (RL)-based model from~\cite{hu2020reinforcement} as a baseline. Their framework can efficiently solve MTSPs with up to a few thousand cities.
Using S-batched REINFORCE, it has shown outstanding generalization ability and has a reduced variance for the gradient estimator than the vanilla REINFORCE~\cite{williams1992simple}.
The other potential baseline methods~\cite{liang2023splitnet, gao2023amarl,park2021schedulenet} either don't provide source codes or cannot be directly applied to our settings.
We also compare iMTSP with a well-known meta-heuristic method implemented by the Google OR-Tools routing library~\cite{kool2018attention, wu2021learning, hu2020reinforcement}.
It provides near-optimal solutions to MTSPs with a few hundred cities and is a popular MTSP baseline~\cite{ORTools_options}.
For abbreviation, we will refer to the two methods as an RL baseline and Google OR-Tools. All approaches are tested locally with the same hardware.

\begin{figure*}
\begin{subfigure}[t]{.33\textwidth}
  \centering
  \includegraphics[width=\linewidth, height=4cm]{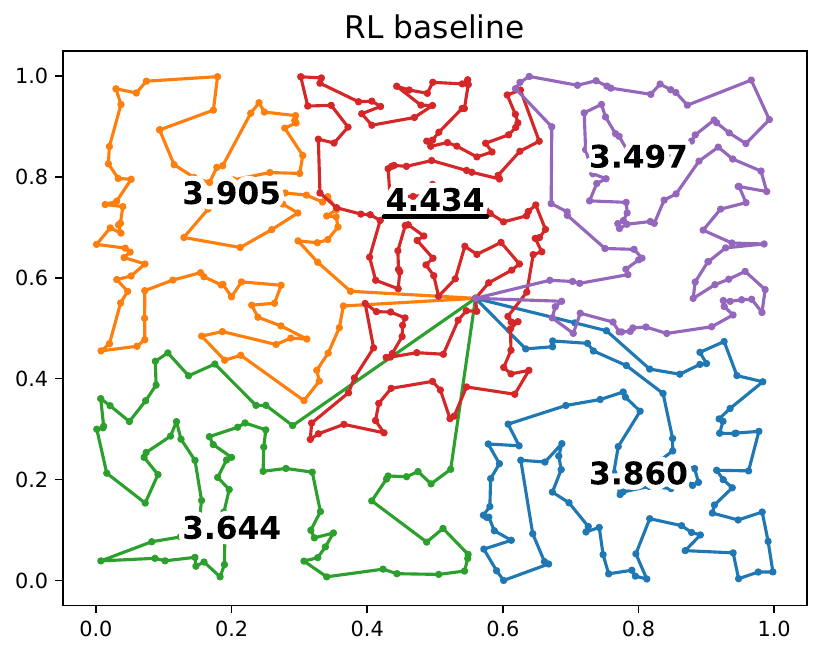}
  \caption{RL baseline on instance \#1.}
  \label{fig:learning_baseling}
\end{subfigure}
\begin{subfigure}[t]{.33\textwidth}
  \centering
  \includegraphics[width=\linewidth, height=4cm]{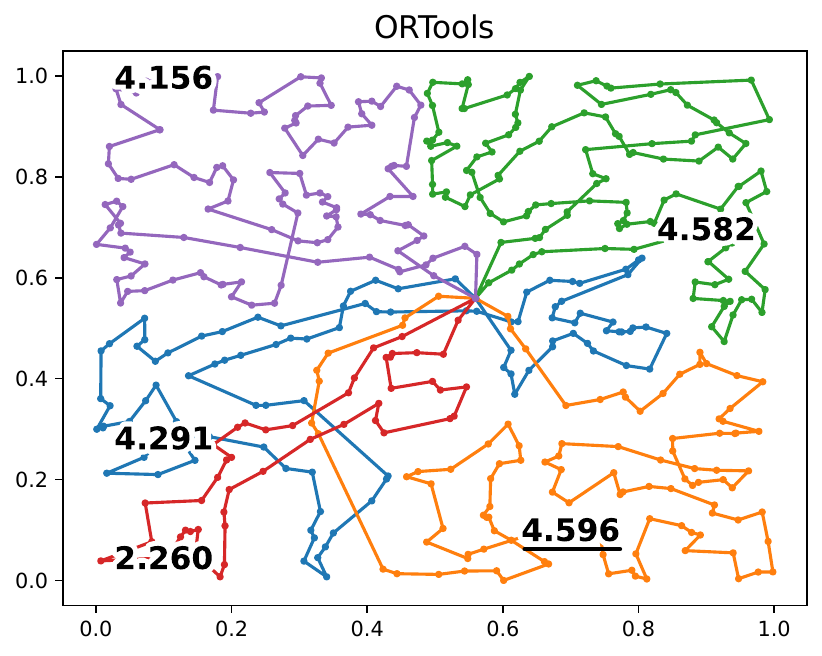}
  \caption{OR-Tools with $1800 \second$ runtime on instance \#1.}
  \label{fig:classic_baseline}
\end{subfigure}
\begin{subfigure}[t]{.33\textwidth}
  \centering
  \includegraphics[width=\linewidth, height=4cm]{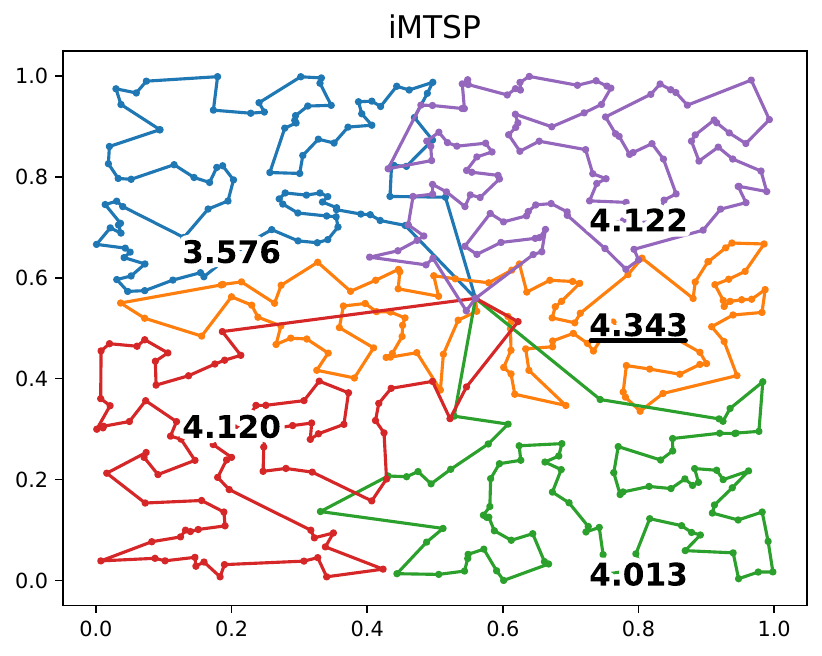}
  \caption{iMTSP on instance \#1.}
  \label{fig:iMTSP}
\end{subfigure}
\begin{subfigure}[t]{.33\textwidth}
  \centering
  \includegraphics[width=\linewidth, height=4cm]{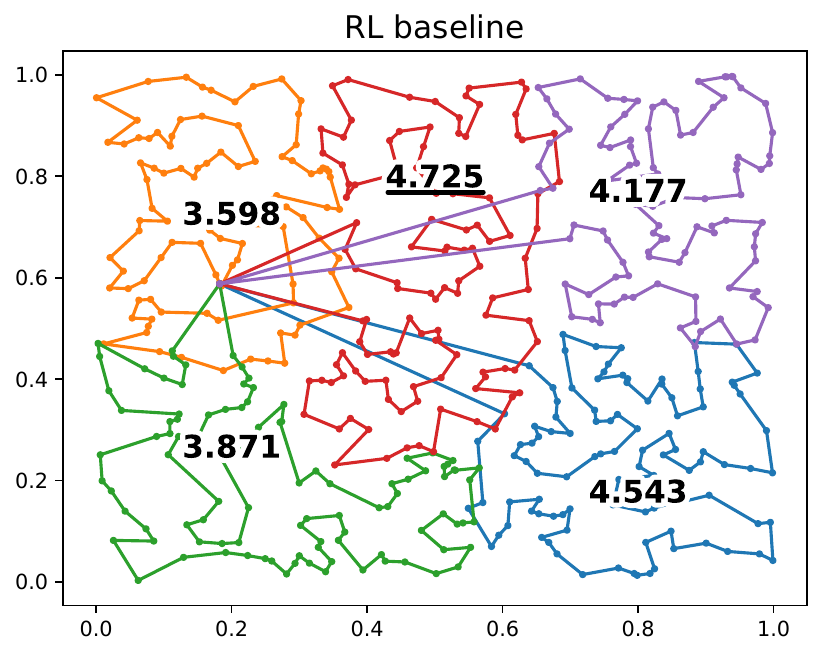}
  \caption{RL baseline on instance \#2.}
  \label{fig:learning_baseling1}
\end{subfigure}
\begin{subfigure}[t]{.33\textwidth}
  \centering
  \includegraphics[width=\linewidth, height=4cm]{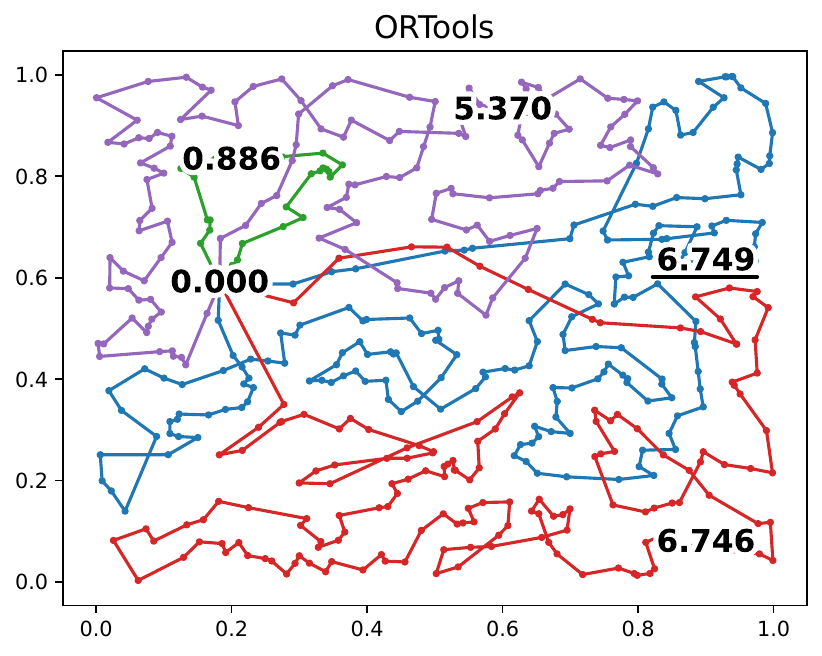}
  \caption{OR-Tools with $1800 \second$ runtime on instance \#2.}
  \label{fig:classic_baseline1}
\end{subfigure}
\begin{subfigure}[t]{.33\textwidth}
  \centering
  \includegraphics[width=\linewidth, height=4cm]{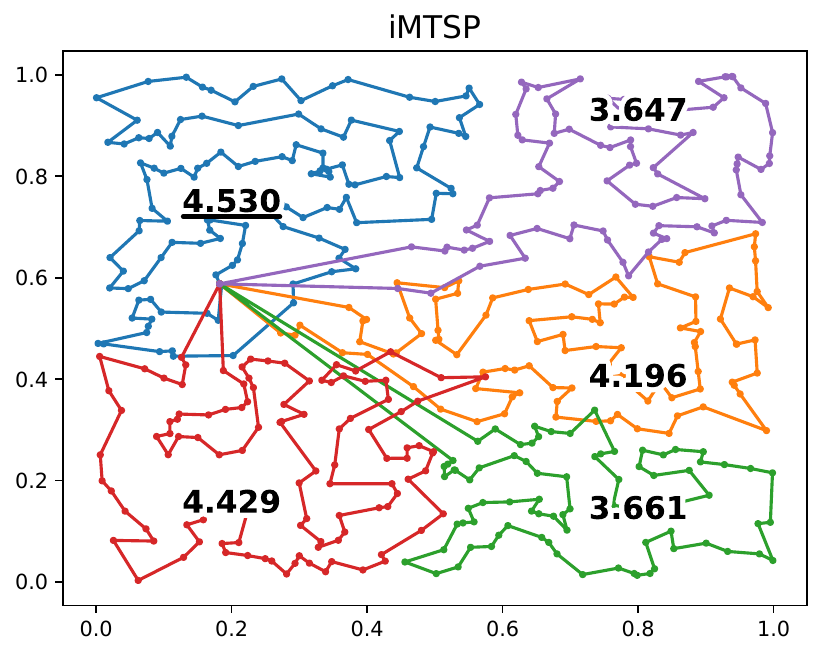}
  \caption{iMTSP on instance \#2.}
  \label{fig:iMTSP1}
\end{subfigure}
\caption{The figures visualize the performance of the baselines and the proposed model on two example MTSPs with $5$ agents and $500$ cities, where the first instance has a central depot and the second instance has an off-center depot. The numbers denote the length of the route, and different colors represent different agents. Longest tour lengths are underlined in each figure. iMTSP always produces the best solution and has fewer sub-optimal patterns like circular partial routes or long straight partial routes.}
\vspace{-10pt}
\label{fig:path}
\end{figure*}

\subsection{Experiment Setting}

We learn from~\cite{vinyals2015pointer, wu2021learning, kool2018attention} that, existing classic and learning-based approaches can handle MTSPs with about $150$ cities, but their performance will significantly compromise with $400$ or more cities. To demonstrate iMTSP's ability to generalization and handle large-scale problems, we conduct experiments on training sets with $50$ to $100$ cities but test the models on problems with $400$ to $1000$ cities. We believe this challenging setting can reflect the generalization ability of the proposed models.

In our tests, all the instances are generated by uniformly sampling $N$ coordinates in a unit rectangular so that both $x$ and $y$ coordinates are in the range of $0$ to $1$.
We employ a batch size of $512$ and a training-to-validation data size ratio of $10:1$. The test set consists of i.i.d. batches but with a larger number of cities ($400$ to $1000$) as aforementioned.

\begin{table}[t]
    \centering
    \caption{Computing time of iMTSP dealing with $1000$ cities. Due to a similar network structure, the RL baseline~\cite{hu2020reinforcement} has roughly the same runtime efficiency compared with ours, while Google OR-Tools~\cite{ORTools_options} cannot find feasible solutions within a similar period.}
    \begin{tabular}{c|ccc}
    \toprule
         Number of agents & 5 & 10 & 15 \\
         \midrule
         Computing time& 4.85&1.98&1.35\\
    \bottomrule
    \end{tabular}
    \label{tab:time}
\end{table}

Specifically, agents and cities are embedded into $64$ dimensional vectors. The allocation network contains one attention layer, where keys, queries, and values are of $128$, $16$, and $16$ dimensions, respectively. The clip constant $\alpha$ in (\ref{imp}) is set to be $10$.
The surrogate network consists of three fully connected layers, whose hidden dimension is $256$ and activation function is \texttt{tanh}. 
Google OR-Tools provides the TSP solver in iMTSP and the MTSP solver as a baseline. It uses the ``Global Cheapest Arc'' strategy to generate its initial solutions and uses ``Guided local search'' for local improvements. All the strategies are defined in~\cite{ORTools_options} and we selected the best options in the experiments.

\subsection{Quantitative Performance}

We next conduct a quantitative performance comparison with the two baseline methods and show it in \tref{tab:table2}. We focus on large-scale MTSPs since many approaches can well solve small-size MTSPs but very few algorithms (both classic or learning-based) can deal with large-scale problems.

\subsubsection{Comparison with Google OR-Tools}
By decomposing a large MTSP into several small TSPs, iMTSP is stronger to handle large-scale problems. We demonstrate such a merit by comparing our model with the classic routing module in OR-Tools, which is given $300$ seconds time budget. The results are shown in Table \ref{tab:table2}. Specially, these route lengths of OR-Tools are averaged over $10$ test instances. OR-Tools appears to be very inferior because it cannot converge to a local minimum given $300$ seconds time budget. In our experiments, OR-Tools can handle $500$ or less cities with $1800$ seconds budget. However, it does not converge even computing for $7200$ seconds with $600$ or more cities, indicating that its computational complexity grows steeply w.r.t. the problem size. Also, the difference of the route length increases as the problem size becomes larger and larger. Remember that the final routes of iMTSP are also solved by Google routing module, while it can easily handle MTSP with up to a few thousand cities with the help of the allocation network.

\begin{figure}[t]
    \centering
    \includegraphics[width=\linewidth]{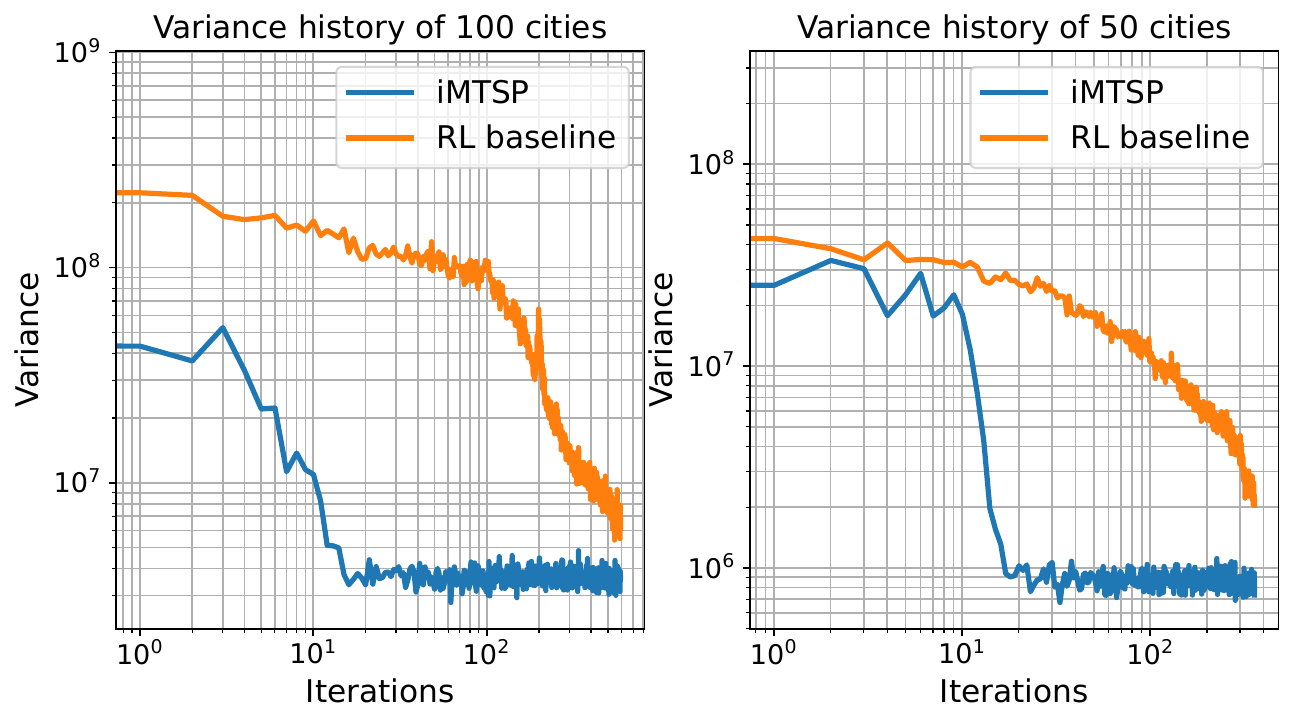}
    \caption{The gradient variance history of our method and the RL baseline during the training process with 50 and 100 cities. The $y$-axis is the sum of the natural logarithm of the variance of all trainable parameters in the allocation network, and $x$-axis denotes the number of iterations. Our gradient estimator converges about $20\times$ faster than the RL baseline.} 
    \label{fig:var history}
\end{figure}

\subsubsection{Comparison with the RL Baseline}
To show the effectiveness of our control variate-based optimization algorithm, we also compare iMTSP with the RL baseline~\cite{hu2020reinforcement} which shares a similar forward process as iMTSP. In most cases, training on bigger problems leads to a better generalization ability on large-scale problems, but it also indicates larger parameter space, more complicated gradient information, and longer training time. This makes the quality of the gradient very crucial. As presented in the \tref{tab:table2}, iMTSP outperforms the RL baseline in most cases, providing solutions with $3.2 \pm 0.01\%$ shorter maximum route length over the RL baseline on average. When the models are trained with $100$ cities, the difference between the route lengths monotonically increases from $0.4\%$ to $8.0\%$ and from $3.4\%$ to $8.9\%$, respectively with $10$ and $15$ agents. The results demonstrate that with the lower-variance gradient provided by our control variate-based optimization algorithm, iMTSP usually converges to better solutions when being trained on large-scale instances. 

\subsection{Qualitative Performance}
To find out the specific advantages in iMTSP's solutions, we plot and compare the results of iMTSP and two baselines on two example instances.
\fref{fig:path} provides a visualization of the solutions from both baselines and iMTSP on two instances with $5$ agents and $500$ cities. The depot of the first instance is close to the center of cities while it's off the center in the second instance. In \fref{fig:learning_baseling1}, there are long straight partial routes (e.g., blue and purple), indicating some agents have travelled for a long distance before they visit their first destination. Also, some circular parts exist in the OR-Tools' solutions (e.g., the green route in \fref{fig:classic_baseline}) meaning that this agent travels for extra distance to visit the cities on the circle. Both of these are typical symptoms of sub-optimality. Such symptoms are not observed in the two solutions from iMTSP. In both instances, iMTSP produces the best performance.
Note that both the RL baseline and iMTSP finish their computation within a few seconds while OR-Tools takes $1800$ seconds to produce comparable solutions.

\subsection{Efficiency Analysis}
Since iMTSP contains both a learning-based allocation network and a classic TSP solver, it is important to identify the bottleneck of the architecture for future improvements.
As in Table \ref{tab:time}, with $5$, $10$, and $15$ agents, the computing time of our model are $4.85$ seconds, $1.98$ seconds, and $1.35$ seconds, respectively, to solve one instance with $1000$ cities. Note that the computing time decrease as the number of agents increases. This fact indicates that the major computational bottleneck is the TSP solvers rather than the allocation network, because more agents are required to run more times of TSP solvers. One of the promising solutions to further reduce our computing time is to create multiple threads to run the TSP solvers in parallel since the TSPs in the lower-level optimization of iMSTP are independent. Since this is not the focus of this paper, we put it as a future work. 

\subsection{Gradient Variance}
As demonstrated in \sref{sec:optimization}, with an optimized surrogate network, the variance of iMTSP's gradient are expected to be smaller than the gradient from the vanilla REINFORCE algorithm~\cite{williams1992simple}. We verify such a hypothesis by explicitly recording the mini-batch gradient variance during the training process. The experiment is conducted twice with $10$ agents visiting $50$ cities and $100$ cities. Training data are divided into several mini-batches, each of which contains $32$ instances. The gradient is computed and stored for every mini-batch, and the variance is later computed with these stored gradient. The network parameters are updated with the average gradient of the whole batch.
The results are shown in \fref{fig:var history}, where the $x$-axis denotes the number of iterations and the $y$-axis represents the natural logarithm of the summed variance of all trainable parameters in the allocation network. It is observed that the gradient variance of our iMTSP is much smaller and converges $20\times$ faster than the RL baseline. This verifies the effectiveness of our control variate-based bilevel optimization process.

\begin{table}[t]
  \centering
    \caption{Results of the ablation studies. The first row is the results from a model optimized by the baseline algorithm while the second line is our iMTSP approach. Our iMTSP approach always produces shorter maximum tour lengths.}
  \resizebox{\linewidth}{!}{
  \begin{tabular}{c|ccccccc}
    \toprule
    \multirow{2}{*}{Optimization}& \multicolumn{6}{c}{\# Test Cities}\\
    \cmidrule{2-8}
        &400 & 500 & 600 & 700 & 800 & 900 & 1000\\
    \midrule
    REINFORCE&3.200&3.477&3.779&4.007&4.249&4.488&4.699\\
    Control Variate&\textbf{3.046}&\textbf{3.273}&\textbf{3.530}&\textbf{3.712}&\textbf{3.924}&\textbf{4.122}&\textbf{4.283}\\
    \bottomrule
  \end{tabular}
  }
  \label{tab:table4}
\end{table}

\subsection{Ablation Study}
To confirm that the improvements of solution quality are majorly contributed by the control variate-based optimization algorithm, we also track the performance of the fine-tuned network with the baseline optimization method. As shown in Table \ref{tab:table4}, the route lengths from our control variate-based optimization algorithm are averagely $7.8\% \pm 2.1\%$ shorter than that from the baseline algorithm. We can also observe that better results always come from iMTSP, which indicates that the better solution quality is mainly contributed by our control variate-based optimization algorithm.

\section{Conclusion and Future Work}\label{sec:conclusion}
In the paper, we reformulate the Min-Max MTSP as a bilevel optimization problem. We introduce a self-supervised framework iMTSP, which can efficiently handle large-scale MTSP with thousands of cities. With the control variate-based optimization algorithm, iMTSP produces a low-variance gradient through a non-differentiable TSP solver and a discrete allocation space to update the allocation network. Experimental results demonstrated the advantages of iMTSP in solution quality, computational efficiency, and optimization stability. In future work, we will also consider the constraints of environmental obstacles and inter-agent collision, which will enable iMTSP to conduct multi-agent path planning in real-world scenarios.

\section{Acknowledgements}
This work was, in part, funded by the DARPA grant DARPA-PS-23-13. The views and conclusions contained in this document are those of the authors and should not be interpreted as representing the official policies, either expressed or implied, of DARPA.

\balance
\bibliographystyle{IEEEtran}
\bibliography{ref}
 
\end{document}